\documentclass[conference]{IEEEtran}
\IEEEoverridecommandlockouts

\usepackage{cite}
\usepackage{amsmath,amssymb,amsfonts}
\usepackage{algorithmic}
\usepackage{graphicx}
\usepackage{textcomp}
\usepackage{xcolor}
\def\BibTeX{{\rm B\kern-.05em{\sc i\kern-.025em b}\kern-.08em
    T\kern-.1667em\lower.7ex\hbox{E}\kern-.125emX}}

\def \R {\mathbb{R}}

\def \Spd {\mathrm{SPD}}

\def \F {\mathrm{F}}

\def \calI {\mathcal{I}}

\def \calO {\mathcal{O}}

\def \calX {\mathcal{X}}

\def \tr {\mathrm{tr}}

\usepackage{amsthm}

\usepackage{url}

\newtheorem{proposition}{Proposition}
\newtheorem{corollary}{Corollary}
\newtheorem{remark}{Remark}

\newtheorem{definition}{Definition}
\newtheorem{theorem}{Theorem}
\usepackage{cleveref} 
\usepackage{mathtools} 
\usepackage{booktabs} 

\begin{document}

\title{Expressivity of congruence-based architectures for DNNs on positive-definite matrices}

\author{\IEEEauthorblockN{Antonin Oswald}
\IEEEauthorblockA{\textit{ICTEAM, UCLouvain} \\
Louvain-la-Neuve, Belgium \\
antonin.oswald@uclouvain.be}
\and
\IEEEauthorblockN{Estelle Massart}
\IEEEauthorblockA{\textit{ICTEAM, UCLouvain} \\
Louvain-la-Neuve, Belgium \\
estelle.massart@uclouvain.be}
}
\maketitle

\begin{abstract} 
This work studies neural architectures for classifying symmetric
positive-definite matrices, focusing on congruence-like layers, in which the
input matrix is multiplied on the left and right by a (possibly rectangular)
weight matrix $\boldsymbol{W}$ and its transpose. Such layers lie at the
core of the celebrated SPDNet and have also been employed independently for dimensionality reduction on positive-definite data. We show that the (semi)-orthogonality constraint commonly imposed on $\boldsymbol{W}$ limits the expressivity of these
layers: for certain activation functions, the resulting architecture collapses
to a one-hidden-layer equivalent. This lack of expressivity follows from a loss of
spectral diversity in congruence-like layers for semi-orthogonal $W$ and is a direct consequence of Poincar\'e's
separation theorem. We then examine the choice of the final classifier,
comparing several Riemannian classifiers and discussing their compatibility with
the feature maps produced by congruence-like layers.
\end{abstract}

\begin{IEEEkeywords}
 Deep learning, Symmetric positive-definite matrices, Riemannian geometry, Machine learning.
\end{IEEEkeywords}

\section{Introduction}
Symmetric and positive-definite (SPD) matrices arise naturally in a wide range of scientific and engineering applications, where they encode essential second-order information such as correlations between signals or covariances of random variables. Applications range from medical imaging and brain-computer interfaces \cite{Barachant2012,Massart2017Inductive} to pedestrian recognition \cite{videoSurveillance},  image set classification \cite{visual_reco_wang, action_reco_ionescu} and radar data processing \cite{cabanes2019}. This ubiquity of SPD matrices motivated the design of machine learning algorithms specifically dedicated to this type of data. In particular, the set of SPD matrices naturally inherits a Riemannian manifold structure \cite{pennec_riemannian_2006}, allowing for the use of Riemannian extensions of usual machine learning techniques.

Various classification methods were proposed for SPD data. The first ones either rely on local linearizations of the manifold \cite{Tuzel2008RiemannianPedestrian}, or assign each point to the class with the closest Fréchet mean, for several choices of metrics \cite{Barachant2012,Massart2017Inductive,cabanes2019}. More recent methods include Riemannian extensions of Gaussian mixture models \cite{said_2017}, kernel functions \cite{mri,rkhqHarandi,steinert2025} and deep neural networks (DNNs)   \cite{classicalSPDNET,nguyen_matrix_2024,nguyen_building_2023, lopez_vector-valued_2021}. 

Among DNNs, the celebrated SPDNet architecture was widely adopted by the community \cite{classicalSPDNET}. This architecture is made of three different types of layers. The first is the BiMap layer, which applies a congruence-like operation to its incoming matrix $X \in \Spd(n)$, where $\Spd(n)$ refers to the manifold of $n \times n$ positive-definite matrices: 
\begin{equation} \label{eq:congruence-like}
    X \mapsto WXW^\top,
\end{equation}
for some full-rank $W^\top \in \R^{n \times p}$ ($p \leq n$). To mitigate optimization instabilities caused by the fact that the set of full-rank $n \times p$ matrices is open \cite[Sec 2.6]{boumal2023intromanifolds}, \cite{classicalSPDNET} enforces a semi-orthogonality constraint on the weight matrix $W$, i.e., $WW^\top = I_p$. The second type of layer is the ReEig layer, which  applies a ReLU activation function on the eigenvalues of the incoming matrix: 
\begin{equation} \label{eq:reeig}
X \mapsto U \phi(\Lambda) U^\top,
\end{equation}
where $X = U \Lambda U^\top $ is an eigenvalue decomposition and $\phi(\Lambda) : \Lambda \mapsto \max(\epsilon I, \Lambda)$ is the ReLU function (shifted by $\epsilon$), applied to the eigenvalues. These two layers are alternated several times, and followed by a LogEig layer, which maps its incoming SPD matrix to a symmetric  matrix by replacing the eigenvalues by their logarithm. This matrix is then fed  into classical Euclidean network layers for classification. 

The main feature of the BiMap and ReEig layers is the fact that they preserve the SPD structure of the data. These layers can thus be seen as a structure-preserving feature extraction block. Note that the congruence-like transformation \eqref{eq:congruence-like} was already proposed as a dimensionality reduction mechanism on SPD matrices \cite{horev2015,harandi_dimensionality_2018}, where the parameter $W$ was trained in order to discriminate as well as possible the data, which could also be seen as a (shallow) structure-preserving feature extraction mechanism. On the other hand, and as highlighted in \cite{Wang_2022_ACCV}, the expressivity of the SPDNet architecture is still poorly understood, which is the focus of this paper.


Note also that, while SPDNet was subsequently endowed with a batch normalization strategy \cite{brooks_riemannian_2019} and is at the core of recent extensions of U-Nets and auto-encoders to SPD data \cite{WANG2023382_unet,boucherie:hal-05600912}, several works proposed concurrent architectures by leveraging the hyperbolic geometric structure of SPD data through gyrovector space theory \cite{nguyen_matrix_2024,nguyen_building_2023, lopez_vector-valued_2021}. The authors of \cite{Chakraborty2018ManifoldNetAD} generalized convolutional neural networks to manifold–valued data by defining convolution via weighted Fréchet means. More broadly, the research on DNN architecture for manifold-valued data is an active topic, and falls under the umbrella of \emph{geometric deep learning} \cite{bronstein_2017}. \\

\noindent \textbf{Contributions:} 
This paper studies the congruence-like transformation
\eqref{eq:congruence-like}, and characterizes how classically recommended constraints on its weight matrices affect expressivity. In \Cref{sec:Ortho,sec:semiortho}, we show that for certain
activation functions, restricting $W$ to be (semi-)orthogonal makes the architecture equivalent to a one-hidden-layer network regardless of depth, a degeneracy that follows from the loss of spectral diversity implied by Poincar\'e's separation theorem. In \Cref{sec:classification}, we then study the final classifier, comparing several candidates, including Riemannian ones, through the discriminatory properties of congruence-like layers under the
associated metrics.

\noindent \textbf{Notations:} $\mathrm{SPD}(n)$, $\mathrm{GL}(n)$ and $\mathcal{O}\left(n\right)$ are respectively the sets of $n \times n$ SPD, invertible and orthogonal matrices. We write $\mathrm{St}(n,p) \coloneqq \left\{Y \in \R^{n \times p} \ \, | \, Y^\top Y = I_p \right\}$, with $p \leq n$, the Stiefel manifold. We introduce the notation $ \calI_{\mathrm{spec}}\left(A\right) \coloneq \left[\lambda_{\min}\left(A\right), \, \lambda_{\max}\left(A\right)\right]$, for the interval between the smallest and largest eigenvalue of $A$. The Frobenius norm is denoted as $\left\| A \right\|_F = \left(\tr \left(A^\top A \right) \right)^{1/2}$. 


 \section{Problem formulation} \label{sec:prob}

 We address the following classification problem. Given a collection of training inputs $X_1, \dots, X_N$, with $X_i \in \Spd(d)$, and associated class label $y_i \in \R^K$ (assumed without loss of generality to be one-hot encoded), our goal is to learn a neural network $\phi : \mathrm{SPD}(d) \to \R^K : X^0 \mapsto \phi(X^0)$, that fits as well as possible these datapoints in regard to some predefined loss function, and whose architecture is characterized by signal propagation equations of the form:
 \begin{equation} \label{eq:NNarch}
    \begin{split}
     X^\ell &= g(W_\ell X^{\ell-1} W_\ell^{\top}), \qquad \ell = 1, \dots L\\
     \phi(X^0) &= f(X^L),
     \end{split}
 \end{equation}
for some weight matrices $W_\ell \in \R^{d^\ell \times d^{\ell-1}}$ for $\ell = 1, \dots L$, with $d^0 = d$, some classifier $f : \Spd(d_L) \to \R^K$ and an activation function $g : \Spd(n) \rightarrow \Spd(n)$, which we assume here to be identical across layers\footnote{While the dimension of the input and output spaces of $g$ vary in general, we keep the same notation throughout the paper for the simplicity of notation.}. We assume here this activation function to be a (primary) matrix function\footnote{Note that this differs from applying a scalar activation function $\sigma : \R \to \R$ to all entries of $A$, which is another way to define matrix functions that we do not consider here.} as defined in \cite{higham_functions_2008}, i.e., to satisfy $g(A) = U g(\Lambda) U^\top$, with $U\Lambda U^\top$ an eigenvalue decomposition of $A$, and
\begin{equation} \label{eq:matrix_activation_function}
    g(\Lambda)_{i,j} = \left \{ \begin{array}{cl}
    0  & \text{if} \; i \neq j\\
    \sigma(\Lambda_{i,i}) & \text{otherwise}, \\
\end{array} \right.
\end{equation}
where $\sigma : \R \to \R$ is a scalar function.  In particular, this encompasses the ReEig layer used in \cite{classicalSPDNET}.

In other words, and as illustrated on \Cref{fig:featureExtractor}, we assume our model to be made of two blocks: a structure-preserving feature extractor block, that alternates between congruence-like transformations and the nonlinear matrix function $g$, and a final classifier $f$. For instance, this classifier could be any Euclidean classifier applied to a mapping of $X^L$ onto a linear space (e.g., through the LogEig layer in \cite{classicalSPDNET}), or any Riemannian classifier (e.g., the minimum-distance-to-the-mean classifier used in \cite{Barachant2012,Massart2017Inductive,cabanes2019}).

  \begin{figure}
     \centering
     \includegraphics[width=\linewidth]{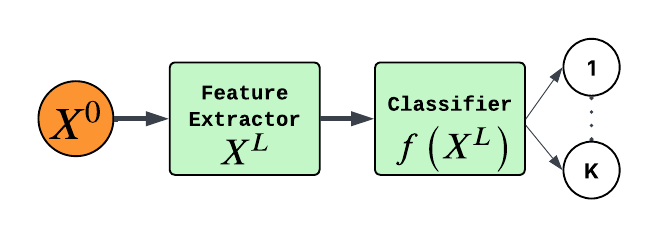}
     \caption{Illustration of the architecture \eqref{eq:NNarch} of depth $L$.}
     \label{fig:featureExtractor}
 \end{figure}

 \section{Model expressivity for orthogonal weight matrices} \label{sec:Ortho}
In this section, we prove that, if all layers are square (i.e., $d^\ell = d$ for $\ell = 1, \dots, L$) and if the weight matrices $W_1, \dots, W_L$ are orthogonal matrices, there exists a one-hidden-layer neural network $\Tilde{\phi} : \Spd(d) \to \R^K$ that is equivalent to \eqref{eq:NNarch}, i.e., such that $\phi(X) = \Tilde{\phi} (X)$ for all $X \in \Spd(d)$. 

 \begin{proposition} \label{prop:orthogonalExpr}
     Assume that $d^\ell = d$ and $W_\ell \in \calO(d)$ for $\ell = 1, \dots, L$ in \eqref{eq:NNarch}. Let us define the one-hidden-layer network $\Tilde{\phi} : \mathrm{SPD}(d) \to \R^K : X^0 \mapsto \Tilde{\phi}(X^0)$  such that
     \begin{equation} \label{eq:NNarchShort}
         \begin{split}
              \tilde X^1 & = g^{(L)}\left( \Tilde{W}_L X^0 \Tilde{W}_L^{\top}\right), \\
             \Tilde{\phi}(X^0) &= f\left( \tilde X^1\right),
         \end{split}
     \end{equation}
     where $\Tilde{W}_L = \left(W_{L}\dots W_2 W_1\right) \in \calO(d)$ and with activation function $g^{(L)}$ being $g(\cdot)$ composed $L$ times with itself. Then, 
     \[ \phi\left(X\right) = \Tilde{\phi}\left(X\right) \qquad \forall X \in \Spd(d).\]
 \end{proposition}
 \begin{proof} 

        Note that there suffices to show that, for all $L$ and for all $X^0 \in \Spd(d)$, 
        \[ X^L = \tilde X^1, \]
        where $X^L$ and $\tilde X^1$ are defined in \eqref{eq:NNarch} and \eqref{eq:NNarchShort}, respectively. The proof is by induction. Trivially, the result holds for $L=1$.
        For the induction step, assuming that the claim holds for a depth $L$, there holds by \eqref{eq:NNarch}
        \begin{equation*}
            \begin{split}
        X^{L+1}&=g\left(W_{L+1} X^L W_{L+1}^{\top} \right) \\
        &= g \left(W_{L+1} g^{(L)}\left( \Tilde{W}_L X^0 \Tilde{W}_{L}^{\top} \right) W_{L+1}^{\top} \right) \\
        & = g^{(L+1)} \left( \Tilde{W}_{L+1}  X^0 \Tilde{W}_{L+1}^{\top}\right), 
        \end{split}
        \end{equation*}
        where the last equality results from our definition of $g$ as matrix function, which implies that $g$ is invariant under congruence with orthogonal matrices; see \cite[Thm. 1.13]{higham_functions_2008}. \qedhere \end{proof}
        
 As a consequence, when all layers in \eqref{eq:NNarch} are square and all weight matrices are orthogonal, the expressivity of the neural network architecture \eqref{eq:NNarch} always reduces to the one of a one-hidden-layer neural network, with depth-depending activation function $g^{(L)}$. The following remark illustrates a possible drop in expressivity for some choices of activation functions.
 
 \begin{remark}
 Assume that $\sigma : \R \to \R$, defined in \eqref{eq:matrix_activation_function}, is a contractive function, i.e., 
 \begin{equation} \label{eq:contrastive}
     |\sigma(x)-\sigma(y)| \leq c |x-y| \qquad \forall x,y \in \R,
 \end{equation} 
 for some positive constant $c < 1$. Then, by the Banach fixed-point theorem, $\lim_{L \to \infty} \sigma^{(L)}(x) = a$ for some unique fixed-point $a \in \mathbb{R}$. Our definition of matrix function $g$ in \eqref{eq:matrix_activation_function} implies then that the spectrum of $X^L$ defined in \eqref{eq:NNarch} becomes more and more uniform as $L$ increases, converging ultimately to some constant value. This will occur for instance when $\sigma$ is the sigmoid, i.e., the mapping $x \mapsto \left(1+\exp (-x) \right)^{-1}$, which satisfies \eqref{eq:contrastive} for  $c= 0.25$.
\end{remark}

 A stronger version of Proposition \ref{prop:orthogonalExpr} can be derived for idempotent activation functions, which satisfy the property $g\left(g\left(X\right)\right) = g \left(X\right)$ for all $X$ in their domain. Note that such idempotent functions can be obtained by choosing $\sigma$ in \eqref{eq:matrix_activation_function} as a scalar idempotent function, which includes the \emph{ReLU} (as in  the ReEig layer in \cite{classicalSPDNET}), and \emph{hard tanh} \cite{list_of_activation_functions}. 

\begin{corollary} \label{corr:orthogonal}
    Under the assumptions of \Cref{prop:orthogonalExpr}, let us define the one-hidden-layer network $\Tilde{\phi} : \mathrm{SPD}(d) \to \R^K : X^0 \mapsto \Tilde{\phi}(X^0)$  such that
     \begin{equation*}
         \begin{split}
             \tilde X^1 &= g\left( \Tilde{W}_L X^0 \Tilde{W}_L^{\top}\right), \\
             \Tilde{\phi}(X^0) &= f\left(\tilde X^1\right),
         \end{split}
     \end{equation*} 
    where $\Tilde{W}_L = \left(W_{L}\dots W_2W_1 \right) \in \calO(d)$. If $g$ is idempotent, then
     \[ \phi\left(X\right) = \Tilde{\phi}\left(X\right) \qquad \forall X \in \Spd(d).\]
\end{corollary}
\begin{proof}
    This results from \Cref{prop:orthogonalExpr} combined with the fact that $g\left(g\left(X\right)\right) = g\left(X\right)$ for all $X \in \mathrm{SPD}(d)$.
\end{proof}
When $g$ is idempotent, the expressivity of \eqref{eq:NNarch} is thus the expressivity of a one-hidden-layer neural network with the same activation function.
\begin{remark}
Although depth does not increase expressivity in this setting, it may still affect optimization and the final accuracy of the model. The resulting overparameterization could act as a preconditioner during training (i.e., accelerating the convergence of the weights), or as a  regularization mechanism; see \cite{arora18} and \cite{arora_2019} for related works in the Euclidean setting. 
\end{remark}


\section{Model Expressivity for semi-orthogonal weight matrices} \label{sec:semiortho}

We next address the case where the weight matrices in \eqref{eq:NNarch} are semi-orthogonal. The following classical result characterizes the action of semi-orthogonal matrices on SPD data through the congruence-like transformation \eqref{eq:congruence-like}. 

\begin{theorem}[Poincaré Separation Theorem, see \cite{bellman_introduction_1997}]

    \label{lemma:poincare}
    
    Let $A$ be an $n \times n$ real symmetric matrix, and $B \in \mathrm{St}(n,p)$. Let $\lambda_j$ be the $j$-th eigenvalue of $A$ for $j = 1,\dots,n$,  and $\mu_i$ the $i$-th eigenvalue of $B^\top A B$ for $i=1,\dots,p$ in descending order. Then,
    $$
        \lambda_i \geq \mu_i \geq \lambda_{i+n-p}, \quad i =1,\dots,p.
    $$
    That is, $\calI_{\mathrm{spec}}\left(B^\top A B\right) \subseteq \calI_{\mathrm{spec}}\left(A\right)$, where $\calI_{\mathrm{spec}} \left ( A \right)$ is the interval between the smallest and largest eigenvalue of $A$. 
\end{theorem}

In other words, the congruence-like transformation \eqref{eq:congruence-like} modifies the spectrum of the matrices in a controlled way, which implies that the expressivity of this architecture does not benefit from depth for some choices of activation functions. 

\begin{proposition} \label{prop:stiefelExprLinear}
 Assume that the scalar activation function $\sigma : \R \to \R$ in \eqref{eq:matrix_activation_function} is the identity on some interval $\calI$, i.e.,  
\[  \sigma(x) = x \qquad \forall x \in \calI,\]
and let $\calX \subseteq \Spd(d) $ be the set of matrices whose spectrum is contained in $\calI$, i.e., 
\[  \calX = \{ X \in \Spd(d) : \calI_{\mathrm{spec}}(X) \subseteq \calI\}.\]
Consider the neural network $\phi$ defined  in \eqref{eq:NNarch}, assuming that $d^0 = d \geq d^1 \geq \dots \geq d^L$, that $W_{\ell}^{\top} \in \mathrm{St}(d^{\ell-1}, d^\ell)$ for all $\ell = 1, \dots, L$ in \eqref{eq:NNarch}, and that $g$ is the matrix activation function associated to $\sigma$ by \eqref{eq:matrix_activation_function}. Let us define the alternative neural network $\Tilde{\phi} : \Spd(d) \to \R^K : X^0 \mapsto \Tilde{\phi}(X^0)$ by
\begin{equation} \label{eq:NNarchShortStiefel}
    \begin{split}
        \tilde X^1 &= \Tilde{W}_L X^0 \Tilde{W}_L^{\top}, \\
        \Tilde{\phi}(X^0) &= f\left(\tilde X^1\right),
    \end{split}
\end{equation}
where $\Tilde{W}_L^{\top} = \left(W_L\dots W_2 W_1\right)^\top \in \mathrm{St}(d, d^L)$. Then,
\[
\phi(X) = \Tilde{\phi}(X), \qquad \forall X \in \calX.
\]
\end{proposition}

\begin{proof}
    As for \Cref{prop:orthogonalExpr}, it is sufficient to prove that for all $X^0 \in \calX$, 
        \[ X^L = \tilde X^1, \]
    where $X^L$ and $\tilde X^1$ are defined in \eqref{eq:NNarch} and \eqref{eq:NNarchShortStiefel}, respectively. The proof is by induction on $L$. We first check the result for $L=1$. Let $X^0 \in \calX$, and let $W^\top \in \mathrm{St}(d,d^1)$. Then, by definition of $\mathcal{X}$, $\calI_{\mathrm{spec}}\left(X^0\right) \subseteq \calI$, and by \Cref{lemma:poincare}, 
    $$\calI_{\mathrm{spec}}\left(WX^0W^\top\right) \subseteq \calI_{\mathrm{spec}} \left(X^0\right) \subseteq \calI.$$  
    This implies that $g\left(WX^0W^\top\right) = WX^0W^\top$, which proves the claim for $L=1$. Assuming that the claim holds for a depth $L$, we prove that it holds for depth $L+1$. Since, by \Cref{lemma:poincare}, 
    $$\calI_{\mathrm{spec}}\left(W_{L+1}X^LW_{L+1}^{\top}\right) \subseteq \calI_{\mathrm{spec}} \left(X^L\right) \subseteq \calI,$$      
    there holds
    \begin{equation*}
            \begin{split}
                X^{L+1} &= g\left(W_{L+1} X^L W_{L+1}^{\top}\right) \\
                &= W_{L+1} X^L W_{L+1}^{\top} \\
                &= \Tilde{W}_{L+1} X^0 \Tilde{W}_{L+1}^{\top}. \qedhere
            \end{split}
        \end{equation*} 
\end{proof}

The following result extends this proposition to an activation function $\sigma : \R \to \R$ whose image is contained in the interval $\calI$ on which it operates as the identity (as the ReLU activation). In this case, all layers except the first one collapse to a single congruence-like transformation of the form \eqref{eq:congruence-like} for all $X \in \Spd\left(d\right)$.

\begin{corollary} \label{corr:rectangular}
    Let the assumptions of \Cref{prop:stiefelExprLinear} hold, and assume that the image of $\sigma$ is included in the interval $\calI$ on which its acts as the identity (note that such a $\sigma$ is a special case of idempotent function). Let $\Tilde{\phi} : \mathrm{SPD}(d) \to \R^K : X^0 \mapsto \Tilde{\phi}(X^0)$  be such that
     \begin{equation*}
         \begin{split}
             \tilde X^1 &=  \Tilde{W}_L g\left(W_1X^0W_{1}^{\top}\right)\Tilde{W}_{L}^{\top}, \\
             \Tilde{\phi}(X^0) &= f\left(\tilde X^1\right),
         \end{split}
     \end{equation*} 
    with $\Tilde{W}_{L}^{\top} = \left(W_{L}\dots W_2W_1\right)^\top \in \mathrm{St}\left(d^2,d^L\right)$ and $L \geq 2$. Then
     \[ \phi\left(X\right) = \Tilde{\phi}\left(X\right) \qquad \forall X \in \Spd(d). \]
\end{corollary}

\begin{proof}
    This follows from \Cref{prop:stiefelExprLinear}, replacing $X^0$ by $X^1$, whose spectrum belongs to $\calI$ due to the assumption on $\sigma$.
\end{proof}

To conclude this section, let us mention that the redundancy of the ReEig layers in SPDNet with (semi-)orthogonal weight matrices was numerically observed in~\cite[Sec.~3.6]{wilsonDeepRiemannianNetworks2025a}.

\section{Interaction between the feature extractor block and the final classifier} \label{sec:classification}

This last section addresses the choice of the final classifier $f$ in \eqref{eq:NNarch}. In the SPDNet architecture, $f$ is a Euclidean classifier applied to the output of the transformation $X^L \mapsto \log \left(X^L\right)$, where $\log$ is the principal logarithm\footnote{The principal logarithm is defined as $\log (A) \coloneq U  \mathrm{diag} \left(  \log \lambda_1, \dots, \log\lambda_{n} \right) U^\top$, for $ A = U \mathrm{diag} \left( \lambda_1, \dots, \lambda_{n}\right)U^\top$ an eigenvalue decomposition.}. An alternative strategy\footnote{Note that SPDNet is directly related to the log-Euclidean metric, as Euclidean distances between the outputs of the $\log$ operator coincide with log-Euclidean distances between the original SPD matrices.} is to classify directly the extracted features using a shallow classifier on the manifold $\Spd(d^L)$, e.g., the Minimum Riemannian Distance to Mean (MRDM) classifier used in \cite{Barachant2012,Massart2017Inductive,cabanes2019}. This classifier assigns each point to the class whose Fréchet mean is the closest, and relies on a notion of distance between two SPD matrices. In this section, we discuss the compatibility between the  choice of the distance and the feature extraction block. We show that several celebrated distances between SPD matrices are invariant under congruence with orthogonal/invertible matrices, and classifiers relying on these distances may thus be less indicated when using the congruence-like transformation \eqref{eq:congruence-like} for feature extraction.

\begin{definition} \label{def:allDistances}
 We recall the definitions of the following classical distances/divergences on  $\mathrm{SPD}(n)$.
\begin{itemize}
    \item The (classical) affine-invariant distance (AI, see  \cite{pennec_riemannian_2006}) between $A, B\in \mathrm{SPD}(n)$ is defined by
        \[ \delta_{\mathrm{AI}} \left(A,B\right) = \left\|\log \left(A^{-\frac{1}{2}}B A^{-\frac{1}{2}}\right) \right\|_F. \]
    \item The log-Euclidean distance (LE,  see \cite{arsigny_geometric_2007}) between $A,B \in \mathrm{SPD}(n)$ is defined by
        \begin{equation*}
            \delta_{\textrm{LE}} \left(A,B\right) = \left\|\log \left(A\right) - \log \left(B\right) \right\|_F.
        \end{equation*}
    \item The Stein divergence (S, see \cite{sra_new_2012}) between $A,B \in \mathrm{SPD}(n)$ is defined as
        \begin{equation*}
            \delta_{\textrm{S}} \left(A,B\right) = \left(\log \det \left(\frac{A+B}{2}\right) - \frac{1}{2} \log \det \left(AB\right) \right)^{\frac{1}{2}}. 
        \end{equation*}
    \item The Bures-Wasserstein distance (BW, see \cite{MasAbs2020})  between $A,B \in \mathrm{SPD}(n)$ is defined as\footnote{\text{The matrix square root is defined as $A ^{\frac{1}{2}} \coloneq U \textrm{diag} \left(\sqrt{\lambda_1}, \dots, \sqrt{\lambda_n} \right) U^\top$,} \text{with $U \textrm{diag} \left(\lambda_1, \dots, \lambda_n \right) U^\top$ and eigenvalue decomposition.}}
    \begin{equation*}
        \delta_{\mathrm{BW}} \left(A,B\right) = \left(\tr \left(A\right) + \tr \left(B\right) - 2 \tr \left( A^{\frac{1}{2}}BA^{\frac{1}{2}} \right)^{\frac{1}{2}} \right)^{\frac{1}{2}}. 
    \end{equation*}
    \item Finally, we recall the Euclidean (E) distance between $A,B \in \mathrm{SPD}(n)$, which is simply defined as
    \begin{equation*}
        \delta_{\mathrm{E}} \left(A,B\right) = \| A - B \|_\F.
    \end{equation*}
\end{itemize}
\end{definition}

We next discuss the (possibly lack of) invariances of these distances under congruence transformations with orthogonal and invertible weight matrices, respectively, and comment on the consequences on the expressivity of the feature extraction blocks relying on the congruence-like transformation \eqref{eq:congruence-like}, when combined with a simple (e.g., MRDM-like) classifier; see \Cref{tab:metrics}. 

\paragraph{Congruence with orthogonal weight matrices} we show that, for all the distances and divergences in \Cref{def:allDistances}, there holds: 
\begin{equation} \label{eq:invar_congr}
    \delta \left(A,B\right) = \delta \left(WAW^\top,WBW^\top\right)
\end{equation} for all $ A,B \in \Spd(n)$ and for all $W \in \calO(n)$. This result is well-known for the log-Euclidean and Euclidean distances, see  \cite[Prop. 3.11]{arsigny_geometric_2007}, and \cite[Section 4.2]{Bhatia1997MatrixAnalysis}, respectively. The affine-invariant distance and the Stein divergence are invariant under congruence with arbitrary (invertible) matrices: for all $W \in \mathrm{GL}(n)$ and $A,B \in \mathrm{SPD}(n)$, $\delta \left(A,B\right) = \delta \left(WAW^\top,WBW^\top\right)$. These results can be found in \cite[Sec. 3.2]{pennec_riemannian_2006} and in \cite[Prop. 2.3]{sra_new_2012} for the affine-invariant distance and the Stein divergence, respectively. As a result, invariance under congruence also trivially holds when $W$ is orthogonal. Finally, the invariance of the Bures-Wasserstein distance is a consequence of the invariance under similarity of the trace and matrix square root.

There follows that, regardless of the nonlinearity, the congruence-like layers \eqref{eq:congruence-like}, with orthogonal weight matrices, do not result in any gain of expressivity: they preserve both the spectrum of the matrices (hence, do not affect the result of successive activation functions, see \Cref{prop:orthogonalExpr}), and the distance separating any two points from the dataset.

\paragraph{Congruence with arbitrary non-singular matrices}
 For $W \in \mathrm{GL}(n)$, \eqref{eq:invar_congr} does not hold for the log-Euclidean, Bures-Wasserstein and Euclidean distances. 
 As stated above, \eqref{eq:invar_congr} still holds for the affine-invariant distance and the Stein divergence, so that, for these two distances, congruence-like layers \eqref{eq:congruence-like} do not allow to separate points from different classes when used in isolation. The expressivity gains of these layers are then limited to the resulting modifications of the spectrum of the matrices before the application of the nonlinearity in the feature extraction block.

\begin{table}
\centering
\begin{tabular}{l|c|c|c|c|c}
\hline
  & \textbf{AI} & \textbf{LE} & \textbf{Stein}& \textbf{BW} & \textbf{E}  \\
\hline
$W \in \mathrm{GL}(n)$ & Yes  & No & Yes & No &  No  \\
$W \in \calO(n)$  & Yes   & Yes & Yes &  Yes & Yes  \\
\hline
\end{tabular}
\vspace{3pt}
\caption{Validity of \eqref{eq:invar_congr} for the distances considered.}
\label{tab:metrics}
\end{table}

\section{Conclusion}

We explored DNN architectures for SPD matrices that involve a structure-preserving feature extraction block made of congruence-like layers and nonlinear matrix functions, followed by a final classifier. We showed that imposing (semi)-orthogonal constraints on congruence weights can collapse deep architectures to one-hidden-layer networks for certain activation function, reducing expressivity. We also noted that several common distances between SPD matrices are invariant under congruences, preventing the resulting classifiers to separate data from different classes. Future work could explore the design of new activation functions to improve expressivity.

\section*{Acknowledgment}

This work was supported by the \emph{GRAVIT-AI} Concerted Research Action (ARC) of the Fédération Wallonie-Bruxelles. E. M. work is also partly funded by the FRS-FNRS Research Project NTTN (grant number TW02223).

\bibliographystyle{plain} 
\bibliography{refs}

@book{higham_functions_2008,
	title = {Functions of {Matrices}},
	url = {https://epubs.siam.org/doi/abs/10.1137/1.9780898717778},
	publisher = {SIAM},
	author = {Higham, N. J.},
	year = {2008},
	doi = {10.1137/1.9780898717778},
}

@inproceedings{classicalSPDNET,
author = {Huang, Z. and Van Gool, L.},
title = {A {R}iemannian network for {SPD} matrix learning},
year = {2017},
booktitle = {Proceedings of the 31st AAAI Conf. on Artificial Intelligence},
pages = {2036–2042},
numpages = {7},
}

@article{list_of_activation_functions,
  title={Activation Functions: Comparison of Trends in Practice and Research for Deep Learning},
  author={Nwankpa, C. and Ijomah, W. and Gachagan, A. and Marshall, S.},
  journal={arXiv preprint arXiv:1811.03378},
  year={2018}
}

@book{bellman_introduction_1997,
	title = {Introduction to {Matrix} {Analysis}, {Second} {Edition}},
	url = {https://epubs.siam.org/doi/abs/10.1137/1.9781611971170},
	publisher = {SIAM},
	author = {Bellman, R.},
	year = {1997},
	doi = {10.1137/1.9781611971170},
}

@article{arsigny_geometric_2007,
	title = {Geometric {Means} in a {Novel} {Vector} {Space} {Structure} on {Symmetric} {Positive}‐{Definite} {Matrices}},
	volume = {29},
	issn = {0895-4798},
	url = {https://epubs.siam.org/doi/10.1137/050637996},
	doi = {10.1137/050637996},
	number = {1},
	urldate = {2026-01-05},
	journal = {SIAM J. on Matrix Analysis and Applications},
	author = {Arsigny, V. and Fillard, P. and Pennec, X. and Ayache, N.},
	year = {2007},
	pages = {328--347},
}

@article{pennec_riemannian_2006,
	title = {A {Riemannian} {Framework} for {Tensor} {Computing}},
	volume = {66},
	issn = {1573-1405},
	url = {https://doi.org/10.1007/s11263-005-3222-z},
	doi = {10.1007/s11263-005-3222-z},
	number = {1},
	journal = {International J. of Computer Vision},
	author = {Pennec, X. and Fillard, P. and Ayache, N.},
	year = {2006},
	pages = {41--66},
}

@article{harandi_dimensionality_2018,
	title = {Dimensionality {Reduction} on {SPD} {Manifolds}: {The} {Emergence} of {Geometry}-{Aware} {Methods}},
	volume = {40},
	issn = {1939-3539},
	shorttitle = {Dimensionality {Reduction} on {SPD} {Manifolds}},
	url = {https://ieeexplore.ieee.org/abstract/document/7822908},
	doi = {10.1109/TPAMI.2017.2655048},
	number = {1},
	journal = {IEEE Tr. on Pattern Analysis and Machine Intelligence},
	author = {Harandi, M. and Salzmann, M. and Hartley, R.},
	year = {2018},
	pages = {48--62},
}

@article{Barachant2012,
  author    = {Barachant, A. and Bonnet, S. and Congedo, M. and Jutten, C.},
  title     = {Multiclass brain-computer interface classification by {R}iemannian geometry},
  journal   = {IEEE Tr. on Biomedical Engineering},
  year      = {2012},
  volume    = {59},
  number    = {4},
  pages     = {920--928},
  doi       = {10.1109/TBME.2011.2172210},
  pmid      = {22010143},
}

@inproceedings{mri,
author = {Jayasumana, S. and Hartley, R. and Salzmann, M. and Li, H. and Harandi, M.},
title = {Kernel {Methods} on the {Riemannian} {Manifold} of {Symmetric} {Positive} {Definite} {Matrices}},
year = {2013},
isbn = {9780769549897},
url = {https://doi.org/10.1109/CVPR.2013.17},
doi = {10.1109/CVPR.2013.17},
booktitle = {Proceedings of the IEEE Conf. on Computer Vision and Pattern Recognition},
pages = {73–80},
numpages = {8},
}

@INPROCEEDINGS{visual_reco_wang,
  author={Wang, R. and Guo, H. and Davis, L. S. and Dai, Q.},
  booktitle={Proceedings of the IEEE Conf. on Computer Vision and Pattern Recognition}, 
  title={Covariance discriminative learning: A natural and efficient approach to image set classification}, 
  year={2012},
  volume={},
  number={},
  pages={2496-2503},
  doi={10.1109/CVPR.2012.6247965}}

@INPROCEEDINGS{action_reco_ionescu,
  author={Ionescu, C. and Carreira, J. and Sminchisescu, C.},
  booktitle={Proceedings of the IEEE Conf. on Computer Vision and Pattern Recognition}, 
  title={Iterated Second-Order Label Sensitive Pooling for 3D Human Pose Estimation}, 
  year={2014},
  volume={},
  number={},
  pages={1661-1668},
  doi={10.1109/CVPR.2014.215}}

@InProceedings{videoSurveillance,
author="Tosato, D.
and Farenzena, M.
and Spera, M.
and Murino, V.
and Cristani, M.",
title="Multi-class {Classification} on {Riemannian} {Manifolds} for {Video} {Surveillance}",
booktitle="Proceedings of the 11th European Conf. on Computer Vision",
year="2010",
pages="378--391",
isbn="978-3-642-15552-9"
}

@book{Bhatia1997MatrixAnalysis,
  title     = {Matrix Analysis},
  author    = {Bhatia, R.},
  series    = {Graduate Texts in Mathematics},
  volume    = {169},
  publisher = {Springer},
  year      = {1997},
  isbn      = {978-0-387-94846-1},
  doi       = {10.1007/978-1-4612-0653-8}
}

@inproceedings{Massart2017Inductive,
  author    = {E. Massart and S. Chevallier},
  title     = {Inductive Means and Sequences Applied to Online Classification of {EEG}},
  booktitle = {Proceedings of Geometric Science of Information},
  year      = {2017},
  doi       = {10.1007/978-3-319-68445-1_88},
  url       = {https://hal.archives-ouvertes.fr/hal-01711499}
}

@article{Tuzel2008RiemannianPedestrian,
  title   = {Pedestrian {Detection} via {Classification} on {Riemannian} {Manifolds}},
  author  = {T{\"u}zel, O. and Porikli, F. and Meer, P.},
  journal = {IEEE Tr. on Pattern Analysis and Machine Intelligence},
  volume  = {30},
  number  = {10},
  pages   = {1713--1727},
  year    = {2008},
  doi     = {10.1109/TPAMI.2008.75},
  pmid    = {18703826}
}

@INPROCEEDINGS{rkhqHarandi,
  author={Harandi, M. and Salzmann, M.},
  booktitle={Proceedings of the IEEE Conf. on Computer Vision and Pattern Recognition}, 
  title={Riemannian coding and dictionary learning: {Kernels} to the rescue}, 
  year={2015},
  volume={},
  number={},
  pages={3926-3935},
  doi={10.1109/CVPR.2015.7299018}}

@inproceedings{brooks_riemannian_2019,
author = {Brooks, D. and Schwander, O. and Barbaresco, F. and Schneider, J.-Y. and Cord, M.},
title = {Riemannian batch normalization for {SPD} neural networks},
year = {2019},
booktitle = {Proceedings of the 33rd International Conf. on Neural Information Processing Systems},
}

@ARTICLE{Chakraborty2018ManifoldNetAD,
  author={Chakraborty, R. and Bouza, J. and Manton, J. H. and Vemuri, B. C.},
  journal={IEEE Tr. on Pattern Analysis and Machine Intelligence}, 
  title={{ManifoldNet}: A {Deep} {Neural} {Network} for {Manifold}-{Valued} {Data} With {Applications}}, 
  year={2022},
  volume={44},
  number={2},
  pages={799-810},
  doi={10.1109/TPAMI.2020.3003846}}

@inproceedings{nguyen_matrix_2024,
title = "Matrix Manifold Neural Networks++",
author = "Nguyen,X.S and Yang, S. and Histace, A.",
year = "2024",
booktitle = "Proceedings of the 12th International Conf. on Learning Representations",
}

@inproceedings{nguyen_building_2023,
	title = {Building {Neural} {Networks} on {Matrix} {Manifolds}: {A} {Gyrovector} {Space} {Approach}},
	shorttitle = {Building {Neural} {Networks} on {Matrix} {Manifolds}},
	url = {https://proceedings.mlr.press/v202/nguyen23f.html},
	language = {en},
	urldate = {2026-01-12},
	booktitle = {Proceedings of the 40th {International} {Conf.} on {Machine} {Learning}},
	author = {Nguyen, X. S. and Yang, S.},
	year = {2023},
}

@Book{boumal2023intromanifolds,
  title     = {An introduction to optimization on smooth manifolds},
  author    = {Boumal, N.},
  publisher = {Cambridge University Press},
  year      = {2023},
  url       = {https://www.nicolasboumal.net/book},
  doi       = {10.1017/9781009166164}
}

@inproceedings{lopez_vector-valued_2021,
	title = {Vector-valued distance and gyrocalculus on the space of symmetric positive definite matrices},
	isbn = {978-1-7138-4539-3},
	urldate = {2026-01-23},
	booktitle = {Proceedings of the 35th {International} {Conf.} on {Neural} {Information} {Processing} {Systems}},
	author = {López, F. and Pozzetti, B. and Trettel, S. and Strube, M. and Wienhard, A.},
	year = {2021},
}

@inproceedings{horev2015,
	title = {Geometry-aware principal component analysis for symmetric positive definite matrices},
	author = {I. Horev and F. Yger and M. Sugiyama},
	year = {2022},
	pages = {1--16},
	booktitle = {Proceedings of the Asian Conf. on Machine Learning},
}

@inproceedings{cabanes2019,
	title = {Toeplitz {Hermitian} positive definite matrix machine learning based in {Fisher} metric},
	author = {Y. Cabanes and F. Barbaresco and M. Arnaudon and J. Bigot},
	year = {2019},
	booktitle = {Proceedings of Geometric Science of Information},
}

@article{said_2017,
	author = {S. Said and L. Bombrun and Y. Berthoumieu and J. H. Manton},
	title = {{Riemannian Gaussian} Distributions on the {Space} of {Symmetric} {Positive} {Definite} {Matrices}},
    journal = {{IEEE} Tr. on Information Theory},
	volume = {63},
	year = {2017},
	pages = {2153--2170},
}

@inproceedings{steinert2025,
	title = {{Universal} {Kernels} via {Harmonic} {Analysis} on {Riemannian} {Symmetric} {Spaces}},
	author = {F. Steinert and S. Said and C. Mostajeran},
	year = {2025},
	booktitle = {Proceedings of Geometric Science of Information},
}

@article{bronstein_2017,
	author = {M. M. Bronstein and J. Bruna and Y. LeCun and A. Szlam and P. Vandergheynst},
	title = {Geometric {Deep} {Learning}: Going beyond {Euclidean} data},
    journal = {{IEEE} Signal Processing Magazine},
	volume = {34},
    number = {3},
	year = {2017},
	pages = {18--42},
}

@inproceedings{sra_new_2012,
	title = {A new metric on the manifold of kernel matrices with application to matrix geometric means},
	booktitle = {Proceedings of the 26th Conf. on Neural Information Processing Systems},
	author = {Sra, S.},
	year = {2012},
}

@inproceedings{arora_2019,
	author = {S. Arora and N. Cohen and W. Hu and Y. Luo},
	title = {{I}mplicit {R}egularization in {D}eep {M}atrix {F}actorization},
	booktitle = {Proceedings of the 33rd Conf. on Neural Information Processing Systems},
	year = {2019},
}

@article{MasAbs2020,
author = {Massart, E. and Absil, P.-A.},
title = {Quotient Geometry with Simple Geodesics for the Manifold of Fixed-Rank Positive-Semidefinite Matrices},
journal = {SIAM J. on Matrix Analysis and Applications},
volume = {41},
number = {1},
pages = {171-198},
year = {2020},
doi = {10.1137/18M1231389},
}

@InProceedings{arora18,
  title = 	 {On the Optimization of Deep Networks: Implicit Acceleration by Overparameterization},
  author =       {Arora, S. and Cohen, N. and Hazan, E.},
  booktitle = 	 {Proceedings of the 35th International Conference on Machine Learning},
  year = 	 {2018},
  url = 	 {https://proceedings.mlr.press/v80/arora18a.html}
}

@article{WANG2023382_unet,
title = {{U}-{SPDN}et: {A}n {SPD} manifold learning-based neural network for visual classification},
journal = {Neural Networks},
volume = {161},
pages = {382-396},
year = {2023},
doi = {https://doi.org/10.1016/j.neunet.2022.11.030},
url = {https://www.sciencedirect.com/science/article/pii/S0893608022004713},
author = {R. Wang and X.-J. Wu and T. Xu and C. Hu and J. Kittler}
}

@InProceedings{Wang_2022_ACCV,
    author    = {Wang, R. and Wu, X.-J. and Chen, Z. and Xu, T. and Kittler, J.},
    title     = {Dream{N}et: A {D}eep {R}iemannian {M}anifold {N}etwork for {SPD} {M}atrix {L}earning},
    booktitle = {Proceedings of the Asian Conf. on Computer Vision},
    year      = {2022},
    pages     = {3241-3257}
}

@article{wilsonDeepRiemannianNetworks2025a,
	title = {Deep {Riemannian} {Networks} for end-to-end {EEG} decoding},
	volume = {3},
	issn = {2837-6056},
	url = {https://doi.org/10.1162/imag_a_00511},
	doi = {10.1162/imag_a_00511},
	journal = {Imaging Neuroscience},
	author = {Wilson, D. and Schirrmeister, R. T. and Gemein, L. A. W. and Ball, T.},
	year = {2025},
}

@inproceedings{boucherie:hal-05600912,
  TITLE = {{SPDNet-AE: a Compact SPD Representation through Riemannian Autoencoding}},
  AUTHOR = {Boucherie, C. and de Surrel, T. and Yger, F.},
  URL = {https://hal.science/hal-05600912},
  BOOKTITLE = {{34th European Symposium on Artificial Neural Networks}},
  YEAR = {2026},
  HAL_ID = {hal-05600912},
  HAL_VERSION = {v1},
}

\end{document}